%% file: paper.tex
\documentclass{article}

\usepackage[preprint,nonatbib]{neurips_2021}

\usepackage[utf8]{inputenc} % allow utf-8 input
\usepackage[T1]{fontenc}    % use 8-bit T1 fonts
\usepackage{hyperref}       % hyperlinks
\usepackage{url}            % simple URL typesetting
\usepackage{booktabs}       % professional-quality tables
\usepackage{amsfonts}       % blackboard math symbols
\usepackage{nicefrac}       % compact symbols for 1/2, etc.
\usepackage{microtype}      % microtypography
\usepackage{xcolor}         % colors
\usepackage{comment}
\usepackage{color}
\usepackage{multirow,setspace}
\usepackage{amsmath,amssymb}
\usepackage{amsthm,latexsym}
\usepackage{bm,amsbsy}
\usepackage{xspace}
\usepackage{enumitem}
\usepackage[normalem]{ulem}
\usepackage{xfrac}
\usepackage{titlesec}
\usepackage{caption}
\usepackage[aboveskip=2pt,belowskip=2pt]{subcaption}
\usepackage{stmaryrd}
\usepackage{algorithm}
\usepackage{algorithmicx}
\usepackage[noend]{algpseudocode}
\usepackage{relsize}

\usepackage{pgf}
\usepackage{tikz}
\usepackage{forest}

\usepackage{multirow,siunitx}

\usetikzlibrary{arrows,decorations.pathmorphing,decorations.footprints,fadings,calc,trees,mindmap,shadows,decorations.text,patterns,positioning,shapes,matrix,fit}
\usetikzlibrary{graphs}%,graphdrawing,quotes}
\usetikzlibrary{decorations.pathreplacing}
\usetikzlibrary{arrows,shadows,backgrounds}
\usetikzlibrary{arrows.meta}
\usetikzlibrary{positioning,fit}
\usetikzlibrary{automata}
\usetikzlibrary{matrix}
\usetikzlibrary{karnaugh}
\usetikzlibrary{shapes.symbols,shapes.misc,shapes.arrows}
\usetikzlibrary{matrix,chains,scopes,decorations.pathmorphing}
\input{defs}

\input{preamble}

\begin{document}

\maketitle

\input{abs}

\input{intro}
\input{prelim}

\input{drs}

\input{dual}

\input{rsdt}
\input{res}
\input{conc}

\input{acks}

%%%%%%%%%%%%%%%%%%%%%%%%%%%%%%%%%%%%%%%%%%%%%%%%%%%%%%%%%%%%

%\clearpage
\bibliographystyle{abbrv}
%\bibliography{refs,xai}

\input{paper.bibl}
\appendix

%\clearpage
\input{xproofs}

%\clearpage
%\input{wplan}

%\clearpage
%\input{exps}

%\clearpage
%\input{dual1}
%%\input{dual2}
%\input{dual3}
%\input{dual2}
%\input{dual3a}
%\input{newdual}

\end{document}

%% file: defs.tex
\newtheorem{proposition}{Proposition}%[section]
\newtheorem{definition}{Definition}%[section]
\newtheorem{example}{Example}%[theorem]
\newenvironment{Proof}{\noindent{\em Proof.~}}{\hfill$\Box$\\[-0.15cm]}

\newcommand{\todoF}[2]{}

\newcommand{\fml}[1]{{\mathcal{#1}}}

\newcommand{\tn}[1]{\textnormal{#1}}

\newcommand{\mbf}[1]{\ensuremath\mathbf{#1}}

\newcommand{\mbb}[1]{\ensuremath\mathbb{#1}}

\newcommand{\Prob}{\ensuremath\tn{Pr}}
\newcommand{\prob}{\Prob}

\newcommand{\mindrs}{Min-$\delta$-Relevant-Set\xspace}
\newcommand{\mincdrs}{Min-C$\delta$-Relevant-Set\xspace}
\newcommand{\minidrs}{Min-I$\delta$-Relevant-Set\xspace}
\newcommand{\bmincdrs}{Min-C$\pmb{\delta}$-Relevant-Set\xspace}
\newcommand{\bminidrs}{Min-I$\pmb{\delta}$-Relevant-Set\xspace}

\newcommand{\majsat}{\ensuremath\tn{MajSAT}}
\newcommand{\ppc}{\ensuremath\tn{PP}}
\newcommand{\nppp}{\ensuremath\tn{NP}^{\tn{PP}}}

\definecolor{gray}{rgb}{.4,.4,.4}
\definecolor{midgrey}{rgb}{0.5,0.5,0.5}
\definecolor{middarkgrey}{rgb}{0.35,0.35,0.35}
\definecolor{darkgrey}{rgb}{0.3,0.3,0.3}
\definecolor{darkred}{rgb}{0.7,0.1,0.1}
\definecolor{midblue}{rgb}{0.2,0.2,0.7}
\definecolor{darkblue}{rgb}{0.1,0.1,0.5}
\definecolor{defseagreen}{cmyk}{0.69,0,0.50,0}

\newcommand{\jnoteF}[1]{}

\newcounter{Comment}[Comment]
\DeclareMathOperator*{\limply}{\rightarrow}

\definecolor{gray}{rgb}{.4,.4,.4}
\definecolor{midgrey}{rgb}{0.5,0.5,0.5}
\definecolor{middarkgrey}{rgb}{0.35,0.35,0.35}
\definecolor{darkgrey}{rgb}{0.3,0.3,0.3}
\definecolor{darkred}{rgb}{0.7,0.1,0.1}
\definecolor{midblue}{rgb}{0.2,0.2,0.7}
\definecolor{darkblue}{rgb}{0.1,0.1,0.5}
\definecolor{darkgreen}{rgb}{0.1,0.5,0.1}
\definecolor{defseagreen}{cmyk}{0.69,0,0.50,0}

\tikzset{
  0 my edge/.style={densely dashed, my edge},
  my edge/.style={-{Stealth[]}},
}

\setlist{nolistsep}

\titlespacing{\section}{0pt}{*1.725}{*1.0}
\titlespacing{\subsection}{0pt}{*1.0}{*0.75}
\titlespacing{\subsubsection}{0pt}{*1.0}{*0.25}
\titlespacing{\paragraph}{0pt}{*0.2125}{1em}
\setlist{nosep}

\newcommand{\axp}{\ensuremath\mathsf{findIDRS}}

%% file: preamble.tex
\title{Efficient Explanations With Relevant Sets}

\author{%
  Yacine Izza\\
  Universit\'{e} de Toulouse, Toulouse, France\\
  \texttt{yacine.izza@univ-toulouse.fr}
  \And
  Alexey Ignatiev\\
  Monash University, Melbourne, Australia\\
  \texttt{alexey.ignatiev@monash.edu}
  \And
  Nina Narodytska\\
  VMware Research, CA, USA\\
  \texttt{nnarodytska@vmware.com}\\
  \And
  Martin C. Cooper\\
  Universit\'{e} Paul Sabatier, IRIT, Toulouse, France\\
  \texttt{martin.cooper@irit.fr}
  \And
  Joao Marques-Silva\\
  IRIT, CNRS, Toulouse, France\\
  \texttt{joao.marques-silva@irit.fr}
}

%% file: abs.tex
\begin{abstract}
  %\begin{comment}  %% Abstract to be submitted
  Recent work proposed $\delta$-relevant inputs (or sets) as a
  probabilistic explanation for the predictions made by a classifier
  on a given input.
  $\delta$-relevant sets are significant because they serve to relate
  (model-agnostic) Anchors with (model-accurate) PI-explanations,
  among other explanation approaches. 
  Unfortunately, the computation of smallest size $\delta$-relevant
  sets is complete for $\nppp$, rendering their computation largely
  infeasible in practice.
  This paper investigates solutions for tackling the practical
  limitations of $\delta$-relevant sets.
  First, the paper alternatively considers the computation of
  subset-minimal sets.
  Second, the paper studies concrete families of classifiers,
  including decision trees among others.
  For these cases, the paper shows that the computation of
  subset-minimal $\delta$-relevant sets is in NP, and can be solved
  with a polynomial number of calls to an NP oracle.
  %
  %Third, the paper also proposes an alternative definition of
  %$\delta$-relevant sets, which yields a polynomial time algorithm
  %for computing subset-minimal (alternative) $\delta$-relevant sets,
  %for several classes of classifiers.
  %
  %Furthermore, the paper generalizes, to the case of
  %$\delta$-relevant sets, a number of minimal hitting set duality
  %relationships between different types of explanations.
  %
  The experimental evaluation compares the proposed approach with
  heuristic explainers for the concrete case of the classifiers
  studied in the paper, and confirms the advantage of the proposed
  solution over the state of the art.
\end{abstract}

%% file: intro.tex
\section{Introduction}
\label{sec:intro}

Recent work proposed $\delta$-relevant inputs (or
sets)~\cite{kutyniok-jair21}, which represent probabilistic
explanations for the predictions made by a classifier given some
input.
$\delta$-relevant sets were shown to generalize both
Anchors~\cite{guestrin-aaai18} and
PI-explanations~\cite{darwiche-ijcai18}, thus revealing a connection
between model-agnostic explanations (e.g.\ Anchors) and model-accurate
explanations (e.g.\ PI-explanations).
Moreover, $\delta$-relevant sets offer a natural solution for
increasing the interpretability of PI-explanations, at the cost of
obtaining intuitive explanations that hold in most, but not all,
points in feature space.
A formidable downside of $\delta$-relevant sets is that their
computation is hard for $\nppp$. This signifies that for most
practical examples, the time for computing minimum $\delta$-relevant
sets will be prohibitive in practice.
To address the computational complexity of finding minimum
$\delta$-relevant sets, a number of solutions can be envisioned.
A first solution is the approximate computation of $\delta$-relevant 
sets. However, for this solution, the formal guarantees offered by
$\delta$-relevant sets may no longer hold.
A second solution is to identify which ML models allow for the
efficient computation of $\delta$-relevant sets.
%, as exemplified by this paper,
%
Finally, a third solution is to investigate possible ways of relaxing
the original definition of $\delta$-relevant
sets~\cite{kutyniok-jair21}.

This paper addresses the second and third solutions listed above.
First, the paper proposes alternative definitions of $\delta$-relevant
sets. Second, the paper analyzes the computation of (different
variants of) $\delta$-relevant sets in the case of decision trees
(DTs).

%However, as this paper also shows, this second solution is not without
%difficulties.
%%
%... endeavour raises a number of challenges, due to the inherent
%difficulty of computing $\delta$-relevant sets.
%%
%This paper addresses the practically effective computation of
%$\delta$-relevant sets, and investigates a number of research
%directions.
%%

Although generally regarded as
interpretable~\cite{breiman-ss01,rudin-naturemi19,molnar-bk20}, recent
work showed that DTs can exhibit \emph{explanation
  redundancy}~\cite{barcelo-nips20,iims-corr20}, i.e.\ there exist DTs
containing paths that are (possibly arbitrarily) longer than a
PI-explanation~\cite{darwiche-ijcai18}. Furthermore, existing
experimental evidence confirms that explanation redundancy is observed
in DTs learned with state of the art DT learners~\cite{iims-corr20}.
Thus, even in the case of DTs, the computation of $\delta$-relevant
sets is of interest when the goal is to improve the interpretability
of ML models.

The main results of the paper can thus be summarized as follows.
First, the paper shows that, for the decision version of computing a
minimum size $\delta$-relevant set (i.e.\ the problem studied in
recent work~\cite{kutyniok-jair21}), is in NP in the case of DTs.
The proof of this result offers a solution for computing a
minimum-size $\delta$-relevant set, in the case of DTs, by using
maximum satisfiability modulo theories
(MaxSMT)~\cite{barrett-hmc18,bjorner-tacas15}.
Second, the paper shows that, in the case of DTs, a relaxed definition
of $\delta$-relevant set enables the computation of (relaxed)
subset-minimal $\delta$-relevant sets in polynomial time.
Third, the paper shows that ML models based on knowledge compilation
(KC) languages~\cite{darwiche-jair02}, including those studied in
recent papers on XAI for KC
languages~\cite{darwiche-ijcai18,darwiche-aaai19,darwiche-ecai20,marquis-kr20,marquis-corr21},
the computation of (relaxed) subset-minimal $\delta$-relevant sets is
also in polynomial time. 
Fourth, the paper shows that recently proposed duality results for
explanations~\cite{inms-nips19,inams-aiia20}, which in practice enable
the enumeration of explanations, can be extended to the setting of
$\delta$-relevant sets.

%Motivated by these results, we consider the computation of relevant
%sets in the case of DTs, and prove that, based on an alternative
%definition of a relevant set, these can be efficiently computed in the
%case of DTs.
%%
%Furthermore, based on the alternative definition of relevant set, we
%show that these can be efficiently computed for a wide range of
%knowledge compilation (KC) languages~\cite{darwiche-jair02}, including
%those studied in recent papers on XAI for KC
%languages~\cite{darwiche-ijcai18,darwiche-aaai19,darwiche-ecai20,marquis-kr20}.

\paragraph{Related work.}
The growing adoption of ML in different settings motivates the recent
interest in
explainability~\cite{muller-dsp18,guidotti-acmcs19,xai-bk19,lipton-cacm18,weld-cacm19,monroe-cacm21}.
Well-known approaches for explaining ML-models are model-agnostic and
based on heuristic
solutions~\cite{guestrin-kdd16,lundberg-nips17,guestrin-aaai18}. These
approaches offer no formal guarantees of
rigor, and practical limitations have been reported in recent
years~\cite{lukasiewicz-corr19,nsmims-sat19,lakkaraju-aies20a,lakkaraju-aies20b}.
% CITE: inms-corr19!!!
%
More recently, model-accurate non-heuristic approaches to
explainability have been
investigated~\cite{darwiche-ijcai18,inms-nips19,darwiche-ecai20,inams-aiia20,marquis-kr20,marquis-corr21,kwiatkowska-ijcai21}.
These non-heuristic approaches are characterized by formal guarantees
of rigor, e.g.\ explanations are valid in any point in feature space.
Unfortunately, non-heuristic methods also exhibit a number of drawbacks,
including for example scalability, explanation size, and the inability
to compute explanations with probabilistic guarantees.
Recent work~\cite{kutyniok-jair21} revealed ways of relating heuristic
and non-heuristic explanations. Our paper builds on this recent work.

\paragraph{Organization.}
The paper is organized as follows.
\autoref{sec:prelim} introduces the notation and definitions used in
the rest of the paper.
\autoref{sec:drs} discusses $\delta$-relevant sets and a number of
alternative definitions.
\autoref{sec:dual} delves into duality between different kinds of
explanations when $\delta$-relevant sets are considered.
\autoref{sec:rsdt} discusses the computation of cardinality-minimal
and subset-minimal $\delta$-relevant sets in the case of decision
trees.
\autoref{sec:res} presents experimental results for computing
$\delta$-relevant sets in the case of DTs.
Finally, \autoref{sec:conc} concludes the paper.

%% file: prelim.tex
\section{Preliminaries}
\label{sec:prelim}

\paragraph{Classification problems \& formal explanations.}
%
%We consider a set of features $\fml{F}=\{1,\ldots,m\}$, where each
%feature takes values from $\mbb{D}=\{0,1\}$. A propositional variable
%$x_i$ is associated with each feature $i$, and the set of propositional
%variables is $X=\{x_1,\ldots,x_m\}$. Feature space is represented by
%$\mbb{F}=\mbb{D}^{m}$.

This paper considers classification problems, which are defined on a
set of features (or attributes) $\fml{F}=\{1,\ldots,m\}$ and a set of
classes $\fml{K}=\{c_1,c_2,\ldots,c_K\}$.
Each feature $i\in\fml{F}$ takes values from a domain $\mbb{D}_i$.
In general, domains can be boolean, integer or real-valued.
%, but in this paper we restrict $\mbb{D}_i=\{0,1\}$ and
%$\fml{K}=\{0,1\}$.
%
%
Feature space is defined as
$\mbb{F}=\mbb{D}_1\times{\mbb{D}_2}\times\ldots\times{\mbb{D}_m}=\{0,1\}^{m}$.
For boolean domains, $\mbb{D}_i=\{0,1\}=\mbb{B}$, $i=1,\ldots,m$, and
$\mbb{F}=\mbb{B}^{m}$.
The notation $\mbf{x}=(x_1,\ldots,x_m)$ denotes an arbitrary point in
feature space, where each $x_i$ is a variable taking values from
$\mbb{D}_i$. The set of variables associated with features is
$X=\{x_1,\ldots,x_m\}$.
Moreover, the notation $\mbf{v}=(v_1,\ldots,v_m)$ represents a
specific point in feature space, where each $v_i$ is a constant
representing one concrete value from $\mbb{D}_i=\{0,1\}$.
An \emph{instance} (or example) denotes a pair $(\mbf{v}, c)$, where
$\mbf{v}\in\mbb{F}$ and $c\in\fml{K}$. (We also use the term
\emph{instance} to refer to $\mbf{v}$, leaving $c$ implicit.)
An ML classifier $\mbb{M}$ is characterized by a \emph{classification
function} $\kappa$ that maps feature space $\mbb{F}$ into the set of
classes $\fml{K}$, i.e.\ $\kappa:\mbb{F}\to\fml{K}$.

%\subparagraph*{Abductive and constrastive explanations.}
%
We now define formal explanations.
Prime implicant (PI) explanations~\cite{darwiche-ijcai18} denote a
minimal set of literals (relating a feature value $x_i$ and a constant
$v_i\in\mbb{D}_i$ %from its domain $\mbb{D}_i$)
that are sufficient for the prediction\footnote{%
PI-explanations are related with abduction, and so are also referred
to as abductive explanations (AXp)~\cite{inms-aaai19}. More recently,
PI-explanations have been studied from a knowledge compilation
perspective~\cite{marquis-kr20,marquis-corr21}, but also in terms of
their computational complexity~\cite{barcelo-nips20}.}.
Formally, given $\mbf{v}=(v_1,\ldots,v_m)\in\mbb{F}$ with
$\kappa(\mbf{v})=c$, a PI-explanation (AXp) is any minimal subset
$\fml{X}\subseteq\fml{F}$ such that,
\begin{equation} \label{eq:axp}
  \forall(\mbf{x}\in\mbb{F}).
  \left[
    \bigwedge\nolimits_{i\in{\fml{X}}}(x_i=v_i)
    \right]
  \limply(\kappa(\mbf{x})=c)
\end{equation}
%%\label{page:eq:axp}
%
AXps can be viewed as answering a `Why?' question, i.e.\ why is some
prediction made given some point in feature space. A different view of
explanations is a contrastive explanation~\cite{miller-aij19}, which
answers a `Why Not?' question, i.e.\ which features can be changed to
change the prediction. A formal definition of contrastive explanation
is proposed in recent work~\cite{inams-aiia20}.
Given $\mbf{v}=(v_1,\ldots,v_m)\in\mbb{F}$ with $\kappa(\mbf{v})=c$, a
CXp is any minimal subset $\fml{Y}\subseteq\fml{F}$ such that,
\begin{equation} \label{eq:cxp}
  \exists(\mbf{x}\in\mbb{F}).\bigwedge\nolimits_{j\in\fml{F}\setminus\fml{Y}}(x_j=v_j)\land(\kappa(\mbf{x})\not=c) %\not\in\fml{Y}
\end{equation}
Building on the results of R.~Reiter in model-based
diagnosis~\cite{reiter-aij87},~\cite{inams-aiia20} proves a minimal
hitting set (MHS, or hypergraph transversal~\cite{berge-bk84}) duality
relation between AXps and CXps, i.e.\ AXps are MHSes of CXps and
vice-versa.
Throughout the paper, $\tn{(M)HS}(\mbb{Z})$ %(resp.~$\tn{MHS}(\mbb{Z})$)
denote the set of (minimal) hitting sets of $\mbb{Z}$.

\paragraph{Relevant sets.}
%\paragraph{$\delta$-relevant sets.}
%
$\delta$-relevant sets were recently proposed~\cite{kutyniok-jair21}
as a formalization of explanation that enables relating different
types of explanation~\cite{kutyniok-jair21}. We briefly overview the
definition of relevant set and associated definitions.
%
%We consider a generalized definition of
%min-$\delta$-relevant-input~\cite{kutyniok-jair21}.
%
The assumptions regarding the probabilities of logical propositions
are those made in earlier work~\cite{kutyniok-jair21}.
Let $\prob_{\mbf{x}}(A(\mbf{x}))$ denote the probability of some
proposition $A$ defined on the vector of variables
$\mbf{x}=(x_1,\ldots,x_m)$, i.e.
\begin{equation} \label{eq:pdefs}
  \begin{array}{ccc}
    \prob_{\mbf{x}}(A(\mbf{x})) =
    \frac{|\{\mbf{x}\in\mbb{F}:A(\mbf{x})=1\}|}{|\{\mbf{x}\in\mbb{F}\}|},
    & \quad &
    \prob_{\mbf{x}}(A(\mbf{x})\,|\,B(\mbf{x})) =
    \frac{|\{\mbf{x}\in\mbb{F}:A(\mbf{x})=1\land{B(\mbf{x})=1}\}|}{|\{\mbf{x}\in\mbb{F}:B(\mbf{x})=1)\}|} 
  \end{array}
\end{equation}

\begin{definition}[$\delta$-relevant set~\cite{kutyniok-jair21}]
  Let $\kappa:\mbb{B}^{m}\to\fml{K}=\mbb{B}$, $\mbf{v}\in\mbb{B}^m$,
  $\kappa(\mbf{v})=c\in\mbb{B}$, and
  $\delta\in[0,1]$. $\fml{S}\subseteq\fml{F}$ is a $\delta$-relevant
  set for $\kappa$ and $\mbf{v}$ if,
  \begin{equation} \label{eq:drs}
    \prob_{\mbf{x}}(\kappa(\mbf{x})=c\,|\,\mbf{x}_{\fml{S}}=\mbf{v}_{\fml{S}})\ge\delta
  \end{equation}
  (where the restriction of $\mbf{x}$ to the variables with indices in
  $\fml{S}$ is represented by
  $\mbf{x}_{\fml{S}}=(x_i)_{i\in\fml{S}}$).
\end{definition}
(Observe that
$\prob_{\mbf{x}}(\kappa(\mbf{x})=c\,|\,\mbf{x}_{\fml{S}}=\mbf{v}_{\fml{S}})$ 
is often referred to as the \emph{precision} of
$\fml{S}$~\cite{guestrin-aaai18,nsmims-sat19}.)
Thus, a $\delta$-relevant set represents a set of features which, if
fixed to some pre-defined value (taken from a reference vector
$\mbf{v}$) ensure that the probability of the prediction being the
same as the one for $\mbf{v}$ is no less than $\delta$. 

\begin{definition}[Min-$\delta$-relevant set] \label{def:mdrs}
  Given $\kappa$, $\mbf{v}\in\mbb{B}^{m}$, and $\delta\in[0,1]$, find
  the smallest $k$, such that there exists $\fml{S}\subseteq\fml{F}$, with
  $|\fml{S}|\le{k}$, and $\fml{S}$ is a $\delta$-relevant set for
  $\kappa$ and $\mbf{v}$.
\end{definition}
With the goal of proving the computational complexity of finding a
minimum-size set of features that is a $\delta$-relevant set, earlier
work~\cite{kutyniok-jair21} restricted the definition to the case
where $\kappa$ is represented as a boolean circuit.
(Boolean circuits were restricted to propositional formulas defined
using the operators $\lor$, $\land$ and $\neg$, and using a set of
variables representing the inputs; this explains the choice of
\emph{inputs} over \emph{sets} in earlier
work~\cite{kutyniok-jair21}.)
Observe that \autoref{def:mdrs} imposes no such restriction on the
representation of the classifier, i.e.\ the logical representation of
$\kappa$ need not be a boolean circuit.

\paragraph{Decision trees.}
A decision tree $\fml{T}$ is a directed acyclic graph having at most
one path between every pair of nodes. $\fml{T}$ has a root node,
characterized by having no incoming edges. All other nodes have one
incoming edge. We consider univariate decision trees where each
non-terminal node is associated with a single feature $x_i$.
Each edge is labeled with a literal, relating a feature (associated
with the edge's starting node) with some values (or range of values)
from the feature's domain. We will consider literals to be of the form
$x_i\in\mbb{E}_i$. $x_i$ is a variable that denotes the value taken
by feature $i$, whereas $\mbb{E}_i\subseteq\mbb{D}_i$ is a subset of
the domain of feature $i$.
The type of literals used to label the edges of a DT allows the
representation of the DTs generated by a wide range of decision tree
learners (e.g.~\cite{utgoff-ml97}).
It is assumed that for any $\mbf{v}\in\mbb{F}$ there exists
\emph{exactly} one path in $\fml{T}$ that is consistent with
$\mbf{v}$. By \emph{consistent} we mean that the literals associated
with the path are satisfied (or consistent) with the feature values in
$\mbf{v}$.
\\ % This *must* be here!!!
%
%\begin{comment}
%\end{comment}

\input{rex}

%% file: rex.tex
\paragraph{Running example.}
%~\\
Throughout the paper, we will consider the decision tree shown
in~\autoref{fig:01} as the running example\footnote{%
  Although the running example considers boolean features (with
  $\mbb{D}_i=\{0,1\}$) and boolean classification, similar conclusions
  would be obtained if we were to consider instead
  real-valued features, e.g.\ having $\mbb{D}_i=[0,1]$.}.

\begin{figure*}[t]
  \input{./texfigs/topfig}
  \caption{DT used as running example} \label{fig:01}
\end{figure*}
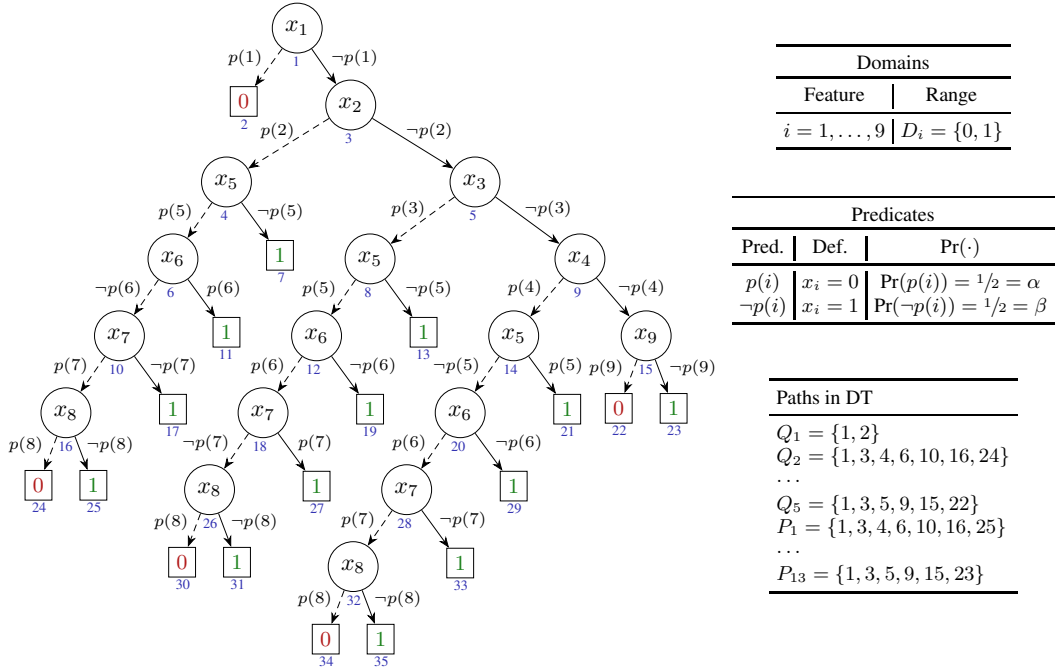

\begin{example} \label{ex:01}
  We consider the example DT from~\autoref{fig:01}. For this
  example and for simplicity, all features are binary with
  $D_i=\{0,1\}$.
  It is also assumed that $\Prob(x_i=0)=\Prob(x_i=1)=\sfrac{1}{2}$,
  which we represent respectively by $\alpha$ and $\beta$, to allow
  other values to be considered.
. Some of the paths in the DT are also shown.
  Moreover, let the instance be 
  $\mbf{v}=(v_1,v_2,v_3,v_4,v_5,v_6,v_7,v_8,v_9)=(1,1,1,1,0,0,0,0,1)$
  with prediction $c=1$. %\oplus
  Since $\mbf{v}$ is consistent with the path ending at node 23, by inspection, we can
  conclude that a possible explanation is $\fml{X}=\{1,2,3,4,9\}$,
  i.e.\ the features listed in the path. It can be shown that this
  corresponds to a PI-explanation.
\end{example}

%% file: texfigs/topfig.tex
\begin{tikzpicture}
  \node (a) at (8.0,3.175)
        {
          \scalebox{0.85}{
            \renewcommand{\arraystretch}{1.05}
            \renewcommand{\tabcolsep}{0.35em}
            %\begin{tabular}{c|c|c} \toprule
            %  Example & Feature & Range  \\ \midrule
            %  \ref{ex:01} & $i=1,\ldots,9$ & $D_i=\{0,1\}$ \\ \midrule
            %  \ref{ex:02} & $i=1,\ldots,9$ & $D_i=[-2,2]$ \\
            %  \bottomrule
            %\end{tabular}
            {\footnotesize%\small
              \begin{tabular}{c|c} \toprule
                \multicolumn{2}{c}{Domains} \\ \toprule
                Feature & Range  \\ \midrule
                $i=1,\ldots,9$ & $D_i=\{0,1\}$ \\ %\midrule
                \bottomrule
              \end{tabular}
            }
          }
        }
        ;
  \node (b) at (1.0,0)
        {
          \hspace*{-0.35cm}\scalebox{0.9}{\input{./texfigs/tree-ex01}}
        }
        ;
  \node (c) at (8.0,1.0)
        {
          \scalebox{0.85}{
            \renewcommand{\arraystretch}{1.05}
            \renewcommand{\tabcolsep}{0.35em}
            {\footnotesize%\small
              \begin{tabular}{c|c|c} \toprule
                \multicolumn{3}{c}{Predicates} \\ \toprule
                Pred. & Def. & $\Prob(\cdot)$ \\ \midrule
                $p(i)$ & $x_i=0$ & $\Prob(p(i))=\sfrac{1}{2}=\alpha$ \\
                $\neg{p(i)}$ & $x_i=1$ & $\Prob(\neg{p(i)})=\sfrac{1}{2}=\beta$ \\
                \bottomrule
              \end{tabular}
            }
          }
        }
        ;
  \node (d) at (8.0,-2.0)
        {
          \scalebox{0.85}{
            \renewcommand{\arraystretch}{1.05}
            \renewcommand{\tabcolsep}{0.35em}
            {\footnotesize%\small
              \begin{tabular}{l} \toprule %|c
                %\multicolumn{2}{c}{Paths in DT} \\ \toprule
                Paths in DT
                %Path $R_j$ %& $\Prob(R_j)$
                \\ \midrule
                $Q_1=\{1,2\}$ \\ %& $\alpha^1$ \\
                $Q_2=\{1,3,4,6,10,16,24\}$ \\ %& $\alpha^4\beta^2$ \\
                %$Q_3=\{1,3,5,8,12,18,26,30\}$ \\ %& $\alpha^4\beta^3$ \\
                %$Q_4=\{1,3,5,9,14,20,28,32,34\}$ \\ %& $\alpha^4\beta^4$ \\[2pt]
                $\cdots$ \\
                $Q_5=\{1,3,5,9,15,22\}$ \\% & $\alpha^1\beta^5$ \\[2pt]
                $P_1=\{1,3,4,6,10,16,25\}$ \\%& $\alpha^3\beta^2$ \\
                $\cdots$ \\
                $P_{13}=\{1,3,5,9,15,23\}$ \\%& $\beta^5$ \\
                \bottomrule
              \end{tabular}
            }
          }
        }
        ;
\end{tikzpicture}

%% file: drs.tex
\section{Complementary Definitions of Relevant Sets}
\label{sec:drs}
%Alternative

Given the prohibitive complexity of solving the \mindrs problem,
this section proposes alternative definitions of minimal relevant
sets, which are shown to yield efficient algorithms for some concrete 
ML models.
First, we consider subset-minimal relevant sets instead of
cardinality-minimal sets. However, we relax the restrictions that
features are boolean and the classification problem is restricted to
two classes.

\paragraph{\bmincdrs{s}.}
Following earlier work on PI-explanations~\cite{darwiche-ijcai18}, we
consider subset-minimal relevant sets.

\begin{definition}[C$\pmb{\delta}$-relevant set]
  Let $\kappa:\mbb{F}\to\fml{K}$, $\delta\in[0,1]$, 
  and instance $(\mbf{v}\in\mbb{F},c\in\fml{K})$.
  $\fml{S}\subseteq\fml{F}$ is a C$\delta$-relevant set for the
  classifier-instance pair, $\kappa$ and $(\mbf{v},c)$, if
  \eqref{eq:drs} holds.
  %
  %\begin{equation} \label{eq:cdr}
  %\prob_{\mbf{x}}(\kappa(\mbf{x})=\kappa(\mbf{v})|\mbf{x}_{\fml{S}}=\mbf{v}_{\fml{S}})\ge\delta
  %\end{equation}
  %
  %(where the restriction of $\mbf{x}$ to the variables with indeces in
  %$\fml{S}$ is represented by
  %$\mbf{x}_{\fml{S}}=(x_i)_{i\in\fml{S}}$).
\end{definition}

(The difference of C$\delta$ to plain $\delta$ relevant sets is that 
$\mbb{F}$ and $\fml{K}$ become unrestricted.)
As noted in earlier work, a (smallest) PI-explanation is a 1-relevant
set for a given pair $\kappa$ and $(\mbf{v},c)$. Furthermore, the main
difference with respect to Anchors~\cite{guestrin-aaai18} is the
assumptions made with respect to sampling. As noted in earlier
work~\cite{kutyniok-jair21}, $\delta$-relevant sets can be related
with different efforts for computing
explanations~\cite{guestrin-aaai18,darwiche-ijcai18,vandenbroeck-ijcai19}.

\begin{definition}[\bmincdrs]
  Let $\kappa:\mbb{F}\to\fml{K}$, $\delta\in[0,1]$, 
  and instance $(\mbf{v}\in\mbb{F},c\in\fml{K})$.
  A \mincdrs is a (subset-)minimal subset $\fml{S}\subseteq\fml{F}$
  that is C$\delta$-relevant for $\kappa$ and $(\mbf{v},c)$.
\end{definition}
(Observe that, in contrast with the definition of
\mindrs~\cite{kutyniok-jair21}, where the objective is to compute a
cardinality-minimal set, the objective of the definition of \mincdrs
it to compute a subset-minimal set.)
For the case where $\kappa$ is implemented as a boolean circuit
(propositional formula defined on the operators $\lor$, $\land$ and
$\neg$), \mindrs is hard for $\nppp$, with the decision problem in
$\nppp$~\cite{kutyniok-jair21}.
Although the complexity of \mincdrs is unknown, we conjecture that
it is similar to the one of \mindrs. Moreover, we have the following
result,
\begin{proposition} \label{prop:cdrs}
  Deciding whether a set $\fml{S}\in\fml{F}$ is a C$\delta$-relevant
  set is \tn{PP}-hard.
\end{proposition}
(The proof in included in the supplementary materials.)
It should be underlined that the high complexity of exactly solving
\mindrs (and \mincdrs) in the general case represents a key practical
limitation.
One additional difficulty with computing a subset-minimal
C$\delta$-relevant set is that its definition is
non-monotone. \eqref{eq:drs} can be written as follows,
\[
\prob_{\mbf{x}}(\kappa(\mbf{x})=c\,|\,\mbf{x}_{\fml{S}}=\mbf{v}_{\fml{S}})
=
\frac{%
  |\{\mbf{x}\in\mbb{F}:\kappa(\mbf{x})=c\land(\mbf{x}_{\fml{S}}=\mbf{v}_{\fml{S}})\}|
}{%
  |\{\mbf{x}\in\mbb{F}:(\mbf{x}_{\fml{S}}=\mbf{v}_{\fml{S}})\}|
}
\]
As the size of set $\fml{S}$ is reduced (e.g.\ as we search for a
minimal set), both the numerator and the denominator can change.
Hence, the value of
$\prob_{\mbf{x}}(\kappa(\mbf{x})=c\,|\,\mbf{x}_{\fml{S}}=\mbf{v}_{\fml{S}})$ 
is not guaranteed to shrink, and in some settings this value can grow.

\paragraph{\bminidrs{s}.}
We show later in the paper that, by considering the probability of the
conjunction of two events instead of the conditional probability, the
resulting monotone definition of relevant set enables more efficient
computation of subset-minimal relevant sets.
One has four possible outcomes when considering two events. In our
case that means we can have:
$[\kappa(\mbf{x})=\kappa(\mbf{v}),\mbf{x}_{\fml{S}}=\mbf{v}_{\fml{S}}]$,
$[\kappa(\mbf{x})=\kappa(\mbf{v}),\mbf{x}_{\fml{S}}\not=\mbf{v}_{\fml{S}}]$,
$[\kappa(\mbf{x})\not=\kappa(\mbf{v}),\mbf{x}_{\fml{S}}=\mbf{v}_{\fml{S}}]$,
and
$[\kappa(\mbf{x})\not=\kappa(\mbf{v}),\mbf{x}_{\fml{S}}\not=\mbf{v}_{\fml{S}}]$.
We are interest in picking sets $\fml{S}$ that minimize the odds of
picking an assignment consistent with $\fml{S}$ and obtaining a
different prediction. Hence, our concern will be to identify sets
$\fml{S}$ that \emph{minimize}
$\Prob(\kappa(\mbf{x})\not=\kappa(\mbf{v}),\mbf{x}_{\fml{S}}=\mbf{v}_{\fml{S}})$.

\begin{definition}[I$\pmb{\delta}$-relevant set]
  Let $\kappa:\mbb{F}\to\fml{K}$, $\delta\in[0,1]$, 
  and instance $(\mbf{v}\in\mbb{F},c\in\fml{K})$.
  $\fml{S}\subseteq\fml{F}$ is a I$\delta$-relevant set for the
  classifier-instance pair, $\kappa$ and $(\mbf{v},c)$, if
  \eqref{eq:idr} holds:
  \begin{equation} \label{eq:idr}
    \prob_{\mbf{x}}(\kappa(\mbf{x})\not=\kappa(\mbf{v}),\mbf{x}_{\fml{S}}=\mbf{v}_{\fml{S}})\le\delta
  \end{equation}
\end{definition}
From the definition of conditional probability (see above in this
section), it is immediate to observe that,
\[
\prob_{\mbf{x}}(\kappa(\mbf{x})\not=\kappa(\mbf{v}),\mbf{x}_{\fml{S}}=\mbf{v}_{\fml{S}})=
\frac{|\{\mbf{x}\in\mbb{F}:\kappa(\mbf{x})\not=c\land(\mbf{x}_{\fml{S}}=\mbf{v}_{\fml{S}})\}|}{|\{\mbf{x}\in\mbb{F}\}|}
\]

\begin{definition}[\bminidrs]
  Let $\kappa:\mbb{F}\to\fml{K}$, $\delta\in[0,1]$, 
  and instance $(\mbf{v}\in\mbb{F},c\in\fml{K})$.
  A \minidrs is a minimal subset $\fml{S}\subseteq\fml{F}$ that is 
  I$\delta$-relevant for $\kappa$ and $(\mbf{v},c)$.
\end{definition}

By observing that for larger sets we can only increase the likelihood
of the function differing from the value of $\kappa(\mbf{v})$, we have
the following result.

\begin{proposition} \label{prop:imono}
  I$\delta$-relevant sets are monotone, i.e.\ for
  $\fml{S}_1\supseteq\fml{S}_2$, it is the case that,
  \[
  \prob_{\mbf{x}}(\kappa(\mbf{x})\not=\kappa(\mbf{v}),\mbf{x}_{\fml{S}_1}=\mbf{v}_{\fml{S}_1})
  \le
  \prob_{\mbf{x}}(\kappa(\mbf{x})\not=\kappa(\mbf{v}),\mbf{x}_{\fml{S}_2}=\mbf{v}_{\fml{S}_2})
  \]
\end{proposition}

\begin{comment}
%
In this work, we study classifiers for which $\mincdrs$ can be solved
more efficiently.
%
For this purpose, we introduce I$\delta$-relevant sets, and the
$\minidrs$ problem, where \eqref{eq:drs} is replaced by,
\begin{equation} \label{eq:idr}
  \prob_{\mbf{x}}(\kappa(\mbf{x})=\kappa(\mbf{v}),\mbf{x}_{\fml{S}}=\mbf{v}_{\fml{S}})\ge\delta
\end{equation}
i.e.\ instead of the conditional probability of the first condition
given the second, we consider the probability of the conjunction of
two conditions.
%
As shown in the paper, the changed definition yields in some cases
key improvements in terms of the computational complexity of computing
subset-minimal I$\delta$-relevant sets.
%
\end{comment}

%\paragraph{Results on duality.}
%~\\

%% file: dual.tex
%\section{Duality Results for Relevant Sets \hfill THIS IS NOT READY!!!!}
% \hfill THIS IS WORK IN PROGRESS!
\section{Duality Results for Relevant Sets}
\label{sec:dual}

Duality results between different types of explanation enable the
implementation of explanation enumeration
algorithms~\cite{inms-nips19,inams-aiia20}\footnote{%
  These recent duality results about explanations build on the work of
Reiter~\cite{reiter-aij87}. In this section, we follow loosely a
recent overview~\cite{slaney-ecai14}.}
This section proves one initial duality result between
$\delta$-relevant sets. Given earlier
work~\cite{inms-nips19,inams-aiia20}, additional results can be
envisioned.

Let $C$ be a predicate, $C:2^{\mbb{F}}\to\{0,1\}$, such that,
\[
C(\fml{S}) = %\triangleq 
[\prob_{\mbf{x}}(\kappa(\mbf{x})\not=c,\mbf{x}_{\fml{S}}=\mbf{v}_{\fml{S}})\le\delta]
\]
We associate with $C$ a set of subsets of $\mbb{F}$,
\[
\mbb{C} = \{\fml{S}\subseteq\mbb{F}\,|\,C(\fml{S})\}
\]
In addition, we define a set of minimal sets,
\[
\mbb{C}_{\tn{min}}=\{\fml{S}\subseteq\mbb{F}\,|\,C(\fml{S})\land\forall(\fml{U}\subsetneq\fml{S}).\neg{C(\fml{U})}\}
\]

Next, we introduce the dual predicate $D$, $D:2^{\mbb{F}}\to\{0,1\}$,
such that,
\[
D(\fml{T}) = %\triangleq 
\neg C(\fml{F}\setminus\fml{T}) =
[\prob_{\mbf{x}}(\kappa(\mbf{x})\not=c,\mbf{x}_{\fml{F}\setminus\fml{T}}=\mbf{v}_{\fml{F}\setminus\fml{T}})>\delta]
\]
The dual of $\delta$-relevant sets are sets $\fml{T}$ of features which if changed entail a change of
class with a probability $>\delta$ and are thus probabilistic analogues of contrastive explanations~\cite{inams-aiia20}. 
As done above, we can define the following sets:
\[
\begin{array}{l}
  \mbb{D} = \{\fml{T}\subseteq\mbb{F}\,|\,D(\fml{T})\}\\[3pt]
  \mbb{D}_{\tn{min}} = \{\fml{T}\subseteq\mbb{F}\,|\,D(\fml{T})\land\forall(\fml{V}\subsetneq\fml{T}).\neg{D(\fml{V})}\}\\
\end{array}
\]

\begin{comment}
%
\[
\begin{array}{l}
  C:2^{\mbb{F}}\to\{0,1\}\\
  C(\fml{S}) = %\triangleq 
  [\prob_{\mbf{x}}(\kappa(\mbf{x})\not=c,\mbf{x}_{\fml{S}}=\mbf{v}_{\fml{S}})\le\delta]\\
  \mbb{C} = \{\fml{S}\subseteq\mbb{F}|C(\fml{S})\}\\
  \mbb{C}_{\tn{min}} = \{\fml{S}\subseteq\mbb{F}|C(\fml{S})\land\forall(\fml{U}\subsetneq\fml{S}).\neg{C(\fml{U})}\}\\
\end{array}
\]

Dual predicate and sets,
\[
\begin{array}{l}
  D:2^{\mbb{F}}\to\{0,1\}\\
  D(\fml{S}) = %\triangleq 
  C(\fml{F}\setminus\fml{S}) =
  [\prob_{\mbf{x}}(\kappa(\mbf{x})\not=c,\mbf{x}_{\fml{F}\setminus\fml{S}}=\mbf{v}_{\fml{F}\setminus\fml{S}})>\delta]\\
  \mbb{D} = \{\fml{T}\subseteq\mbb{F}|D(\fml{S})\}\\
  \mbb{D}_{\tn{min}} = \{\fml{T}\subseteq\mbb{F}|D(\fml{T})\land\forall(\fml{V}\subsetneq\fml{T}).\neg{D(\fml{V})}\}\\
\end{array}
\]
%
\end{comment}

%Let $\tn{(M)HS}(\mbb{Z})$ %(resp.~$\tn{MHS}(\mbb{Z})$)
%denote the set of (minimal) hitting sets of $\mbb{Z}$. Then,
Given the above, monotonicity of the predicates $C$ and $D$
(see~\autoref{prop:imono}), allows us to prove the following results,
\begin{proposition} \label{prop:dual}
  $\mbb{C}=\tn{HS}(\mbb{D})$, 
  $\mbb{D}=\tn{HS}(\mbb{C})$,
  $\mbb{C}_{\tn{min}}=\tn{MHS}(\mbb{D})$, and
  $\mbb{D}_{\tn{min}}=\tn{MHS}(\mbb{C})$.
%  $\mbb{C}_{\tn{min}}=\tn{MHS}(\mbb{D}_{\tn{min}})$, and
%  $\mbb{D}_{\tn{min}}=\tn{MHS}(\mbb{C}_{\tn{min}})$.
\end{proposition}

\begin{proof}
First, we consider $\mbb{C}=\tn{HS}(\mbb{D})$.
Suppose, there exists $\fml{S}$ such that it is not a hitting set of sets in $\mbb{D}$. 
Namely, $\fml{S}$ does not hit some set $\fml{T}\in \mbb{D}$, $\fml{S} \cap \fml{T} = \emptyset$.
By definition, $\fml{S} \subset \fml{F} \setminus \fml{T}$. 
We recall that a predicate $P$ is \emph{monotone} if for all $S,S' \subseteq \fml{F}$,
\[ S \subset S' \land P(S) \ \ \rightarrow \ \ P(S').
\]
Hence, as  $\fml{S} \subset \fml{F} \setminus \fml{T}$ and $C(S)$ holds, $C(\fml{F} \setminus \fml{T})$ must hold.
This leads to a contradiction as $D(\fml{T}) = \neg C(\fml{F} \setminus \fml{T})$  by definition.  The reverse direction, $\mbb{D}=\tn{HS}(\mbb{C})$, is similar.

Second, we consider $\mbb{C}_{\tn{min}}=\tn{MHS}(\mbb{D})$.
The proof that $\fml{S} \in \mbb{C}_{\tn{min}}$ is a hitting set of $\mbb{D}$ follows from the argument above as $\mbb{C}_{\tn{min}} \subseteq \mbb{C}$.
% and $\mbb{D}_{\tn{min}} \subseteq \mbb{D}$.
%
Next, we suppose $\fml{S} \in \mbb{C}_{\tn{min}}$ is \emph{not a minimal} hitting set of $\mbb{D}$. 
Let $\fml{G} \subset \fml{S}$ be the minimal hitting set.
By definition of minimality of $\fml{S}$, $\neg C(\fml{G})$  must hold. 
Consider $\fml{W}$ such that $\fml{G} = \fml{F} \setminus \fml{W}$.
We have that $\neg C(\fml{G})=  \neg C(\fml{F} \setminus \fml{W}) = D(\fml{W})$. Therefore, $\fml{W}\in \mbb{D}$. 
As $\fml{G} \cap \fml{W} = \emptyset$ by construction, 
 $\fml{G}$ does not hit $\fml{W}\in \mbb{D}$ and it is not a hitting set.
%There is a minimal subset of $\fml{W}$, $\fml{W}'$, such that $D(\fml{W}')$ holds,
%so $\fml{W}' \in \mbb{D}_{\tn{min}}$.
%Hence, $\fml{G}$ does not hit $\fml{W}'$ and it is not a hitting set.
%
The reverse direction, $\mbb{D}_{\tn{min}}=\tn{MHS}(\mbb{C})$, is similar.
\end{proof}

%% file: rsdt.tex
\section{Relevant Sets for DTs \& Other Classifiers}
\label{sec:rsdt}

This section shows that the decision problem for $\delta$-relevant
(and so C$\delta$-relevant) sets is in NP when $\kappa$ is represented
by a decision tree\footnote{%
  For simplicity, the paper considers the case of non-continuous
  features. However, in the case of DTs, the results generalize to
  continuous features.}.
Thus, \mincdrs can be solved with at most a logarithmic number of
calls to an NP oracle. (This is true since we minimize on the number
of features.)
This section also shows the decision problem for I$\delta$-relevant
sets is in P. Thus, the \minidrs can be solved in polynomial time in
the case of DTs.
Furthermore, the section extends the previous results to the case of
knowledge compilation (KC) languages~\cite{darwiche-jair02}.

\paragraph{Path probabilities for DTs.}
Let $\mbf{v}\in\mbb{F}$ and suppose that 
$\kappa(\mbf{v})=c\in\fml{K}$.
For a DT $\fml{T}$, let $\fml{P}=\{P_1,\ldots,P_M\}$ denote the paths
with prediction $c$, and let $\fml{Q}=\{Q_1,\ldots,Q_N\}$ denote the
paths with a prediction in $\fml{K}\setminus\{c\}$.
Let $\fml{R}=\fml{P}\cup\fml{Q}$ denote the set of all paths in the DT
$\fml{T}$.
%We assume that DTs respect the condition that, for any
%$\mbf{u}\in\mbb{F}$, there exists exactly one path in $\fml{T}$ that
%is consistent with $\mbf{u}$.
%
%
For $R_j\in\fml{R}$, let $||R_j||$ denote the number of points in
$\mbb{F}$ consistent with $R_j$. Thus, the path probability of any
path $R_j\in\fml{R}$ is,
$\prob(R_j)=\sfrac{||R_j||}{||\mbb{F}||}$. (The path probability of
some tree path $R_j$ is the empirical probability of a point in
feature space chosen at random being consistent with the path $R_j$.)
As a result, we get,
\[
\sum\nolimits_{R_j\in\fml{P}}\prob(R_j)+
\sum\nolimits_{R_j\in\fml{Q}}\prob(R_j)=
%\mbb{P}+\mbb{Q}=
1
%\sum_{R_j\in\fml{R}}\prob(R_j)=\sum_{R_j\in\fml{R}}\frac{||R_j||}{||\mbb{F}||}=1
\]
%(Observe that for DTs, it is simple to extend these definitions to
%real-valued features, but that is beyond the scope of this paper.)
%
Moreover, $||\mbb{F}||=||\mbb{D}_1||\times\cdots\times||\mbb{D}_m||$.
For each path $R_j$, let $d_i$ denote the number of values in
$\mbb{D}_i$ that is consistent with the literals defined on $x_i$ in
path $R_j$. Thus, $||R_j||=d_1\times\cdots\times{d}_m$.

%In addition, we introduce $\mbb{Q}$ as follows,
%\[
%\mbb{Q}=\sum_{{Q}_j\in\fml{Q}}\Prob(Q_j)
%\]

\begin{table}[t]
  %\begin{center}
  \hspace*{-0.20cm}
    \scalebox{0.85}{
      \renewcommand{\tabcolsep}{0.2735em}
      {\footnotesize
        \begin{tabular}{c|ccccc|ccccccccccccc} \toprule
          Path $R_j$ &
          $Q_1$ & $Q_2$ & $Q_3$ & $Q_4$ & $Q_5$ & 
          $P_1$ & $P_2$ & $P_3$ & $P_4$ & $P_5$ &
          $P_6$ & $P_7$ & $P_8$ & $P_9$ & $P_{10}$ &
          $P_{11}$ & $P_{12}$ & $P_{13}$
          \\ \toprule
          $\Prob(R_j)$  & %Probability
          $\alpha^1$ & % Q1
          $\alpha^4\beta^2$ & % Q2
          $\alpha^4\beta^3$ & % Q3
          $\alpha^4\beta^4$ & % Q4
          $\alpha^1\beta^4$ & % Q5
          $\alpha^3\beta^3$ & % P1
          $\alpha^2\beta^3$ & % P2
          $\alpha^3\beta^1$ & % P3
          $\alpha^1\beta^2$ & % P4
          $\alpha^4\beta^3$ & % P5
          $\alpha^4\beta^2$ & % P6
          $\alpha^3\beta^2$ & % P7
          $\alpha^1\beta^3$ & % P8
          $\alpha^3\beta^5$ & % P9 3,5
          $\alpha^2\beta^4$ & % P10 2,4
          $\alpha^1\beta^5$ & % P11
          $\alpha^2\beta^3$ & % P12
          $\beta^5$ % P13
          \\ \bottomrule
        \end{tabular}
      }
    }
    %\end{center}
    \smallskip
  \caption{Path probabilities for running example} \label{tab:02}
\end{table}

\begin{example} \label{fig:02}
  For the example in~\autoref{fig:01}, \autoref{tab:02} shows the path
  probability for each path in the DT, computed using the above
  definition of path probability.
  %(...) in~\autoref{fig:ex01d}, which also summarizes the path
  %probabilities.
\end{example}

\paragraph{\bmincdrs{s} for DTs.}
Whereas in the general case, deciding whether there exists a
$\delta$-relevant set of size no greater than $k$ is complete for
$\tn{NP}^{\tn{PP}}$~\cite{kutyniok-jair21}, in the the case of DTs,
one can prove that this decision problem is in NP (and the same
applies in the case of a subset-minimal C$\delta$-relevant set).

\begin{proposition}
  For DTs, given $\mbf{v}\in\mbb{F}$, with $\kappa(\mbf{v})=c\in\fml{K}$,
  deciding the existence of min-$\delta$-relevant set of size at most
  $k$ is in NP.
\end{proposition}

\begin{Proof}
  We reduce the problem of deciding the existence of a
  min-$\delta$-relevant set of size at most $k$ to the decision
  version of the maximum satisfiability modulo theories (SMT)
  problem~\cite{barrett-hmc18,bjorner-tacas15} (assuming a suitable
  quantifier-free theory).

  Let $s_i$ be a boolean variable such that $s_i=1$ iff $i\in\fml{F}$
  is included in the $\delta$-relevant set. %$x_i=v_i$ 
  Moreover, let $t_j$ be a boolean variable, such that $t_j=1$ iff
  path $R_j\in\fml{P}\cup\fml{Q}$ is inconsistent, i.e.\
  %(with prediction other than $c$)
  some feature $i$ added to the $\delta$-relevant set makes $R_j$
  inconsistent.
  Thus, if path $R_j$ is inconsistent with the value given to feature
  $i$, then it must be the case that,
  \[
  s_i\limply{t_j}
  \]
  Also, if $R_j$ is deemed inconsistent, then it must be the case
  that,
  \[
  t_j\limply\bigvee\nolimits_{i\in{I_j}}{s_i}
  \]
  where $i$ ranges over the set of features $I_j$ that make $R_j$
  inconsistent, given $\mbf{v}$.

  The set of picked features $\fml{S}$ , %i.e.\ those for which $s_i=1$
  (i.e.\ $\fml{S}$ is the set of features having $s_i=1$), ensures that
  \[
  \Prob_{\mbf{x}}(\kappa(\mbf{x})=\kappa(\mbf{v})|\mbf{x}_{\fml{S}}=\mbf{v}_{\fml{S}})\ge\delta
  \]
  From which we get,
  \[
  \begin{array}{l}
    \Prob_{\mbf{x}}(\kappa(\mbf{x})=\kappa(\mbf{v})\land\mbf{x}_{\fml{S}}=\mbf{v}_{\fml{S}})\ge\delta\times\Prob_{\mbf{x}}(\mbf{x}_{\fml{S}}=\mbf{v}_{\fml{S}})
    \Leftrightarrow
    \\[3pt]
    \sum_{j,R_j\in\fml{P}}\neg{t_j}\times\Prob(R_j)\ge
    \delta\times\sum_{j,R_j\in\fml{P}\cup\fml{Q}}\neg{t_j}\times\Prob(R_j)\\
  \end{array}
  \]
  which is a linear inequality on the $t_j$ variables, since the
  probabilities are constant.
  An additional constraint is that the number of $s_i$ variables
  assigned value 1 cannot exceed $k$, i.e.\ the bound on the size of
  the relevant set $\fml{S}$:
  \[
  \sum\nolimits_{i\in\fml{F}}{s_i}\le{k}
  \]
  which is another linear inequality, this one on the $s_i$ variables.
  By conjoining all the constraints, and assignment to the $s_i$ and
  $t_j$ variables that satisfies the constraints picks a
  $\delta$-relevant set whose size does not exceed $k$.
  \qedhere
\end{Proof}

Clearly, since the decision problem is in NP, it is immediate how to
compute a cardinality minimal $\delta$-relevant set by binary search
on the number of $s_i$ variables included in the set. Since the number
of variables equals the size of $\fml{F}$, then we are guaranteed to
need (in the worst-case) a logarithmic number of calls to an NP
oracle.
Furthermore, since the decision problem for the min-$\delta$-relevant
problem is in NP, it is also the case that the decision problem for
the min-C$\delta$-relevant problem is also in NP.
Finally, we conjecture that the min-$\delta$-relevant set, but also
the min-C$\delta$-relevant problem are both hard for NP. These
conjectures are further justified below.

\paragraph{\bminidrs{s} for DTs.}
One apparent reason to the conjectured complexity is the fact that the
conditional probability used for defining $\delta$-relevant and
C$\delta$-relevant sets is non-monotone. As a result, earlier in the
paper we introduced I$\delta$-relevant sets, which were shown to be
monotone in~\autoref{prop:imono}.
We now show that, in the case of DTs, computing a subset-minimal
I$\delta$-relevant set is in P.
%
%
%By analyzing the apparent reason for the complexity of solving the $
%For a DT, if \eqref{eq:idr} is used instead of \eqref{eq:drs}, then a
%\minidrs set can be computed in polynomial time.
%
%
The criterion for a set of features to be I$\delta$-relevant is:
\[
\Prob_{\mbf{x}}(\kappa(\mbf{x})\not=\kappa(\mbf{v}),\mbf{x}_{\fml{S}}=\mbf{v}_{\fml{S}})\le\delta
\]
Observe that this constraint holds when $\fml{S}=\fml{F}$ and,
by~\autoref{prop:imono} I$\delta$-relevant sets are monotone. As a
result, we can compute a subset minimal I$\delta$-relevant set, as
proposed in\autoref{alg:1axp}. (The algorithm is standard, and can be
traced to at least the work of
Chinneck \& Dravnieks~\cite{chinneck-jc91}.
The novelty is its use for computing min-I$\delta$-relevant sets.) 
\begin{algorithm}[t]
  \input{./algs/oneaxp}
  \caption{Finding one min-I$\delta$-relevant set (IDRS)} \label{alg:1axp}
\end{algorithm}
The algorithm maintains an invariant representing the assertion that
the set $\fml{S}$ is a I$\delta$-relevant set.
Initially, all features are included in set $\fml{S}$,
i.e.\ $\fml{S}=\fml{F}$, and so $\fml{S}$ is a I$\delta$-relevant
set.
The algorithm then iteratively removes one feature at a time, and
checks whether the invariant is broken. If it is, then the feature is
added back to set $\fml{S}$. Otherwise, we are guaranteed, by
monotonicity, that for \emph{any} superset of set $\fml{S}$, the
invariant holds.

\begin{example}
  We consider the running example (see~\autoref{fig:01}, with instance
  $(\mbf{v},c)$ given by
  $\mbf{v}=(v_1,v_2,v_3,v_4,v_5,v_6,v_7,v_8,v_9)=(1,1,1,1,0,0,0,0,1)$
  and $c=\kappa(\mbf{v})=1$. %\oplus
  As argued earlier, by setting $\delta=0$, one AXp is
  $\fml{X}=\{1,2,3,4,9\}$.
  Let
  $\epsilon(\fml{S})=\Prob_{\mbf{x}}(\kappa(\mbf{x})=\kappa(\mbf{v}),\mbf{x}_{\fml{S}}=\mbf{x}_{\fml{S}})$,
  denoting the error associated with some set of features $\fml{S}$.
  With the purpose of improving the interpretability of the
  explanation, we set $\delta=0.03$, and work towards finding an
  explanation with fewer literals.\\
  To illustrate the execution of the algorithm, we assume that the
  features are analyzed in the order $\langle1,2,3,4,9\rangle$.
  \autoref{tab:03} summarizes the execution of the algorithm.
  The algorithm first analyzes dropping feature $1$ from the
  explanation $\fml{X}$. In this case, only path $Q_1$ can be
  made consistent. Given that the probability of picking an assignment
  consistent with $Q_1$ is $0.5$, then feature 1 cannot be dropped
  from the explanation, as that would put the error about the target 
  threshold.
  Next, the algorithm considers dropping feature $2$ from the
  explanation. In this case, only path $Q_2$ can be made consistent.
  Given that the probability of picking an assignment consistent with
  $Q_2$ is $(\sfrac{1}{2})^6=0.015625$, then are still below the
  target absolute fraction of error of $\delta=0.03$.
  %%the precision is still $0.984375$.
  Hence, we remove feature 2 from the explanation.
  For feature 3, and since feature 2 is already dropped, then both
  paths $Q_2$ and $Q_3$ can be made consistent. In this case, the
  error raises to 0.0234, but it is still below 0.3, and so feature 3
  is also dropped from the explanation.
  A similar analysis allows concluding that feature 4 can also be
  dropped from the explanation.
  In contrast, after removing features 2, 3 and 4, feature 9 cannot be
  dropped form the explanation. The resulting approximate explanation
  (i.e.\ a I$\delta$-relevant set) is thus $\{1,9\}$. Moreover, the
  explanation $\{1,9\}$ ensures that, in more than 97\% of the points
  in feature space consistent with the values of features 1 and 9, the
  prediction will be the intended one, i.e.\ 1.
\end{example}

\begin{table}[t]
  \smallskip
  \begin{center}
    \begin{tabular}{c|c|c|c|c|c} \toprule
      $\fml{S}$ & $i$ & $\fml{R}=\fml{S}\setminus\{i\}$ & $\fml{Q}$
      paths consistent with $\fml{R}$
      & $\epsilon(\fml{R})$ %\fml{S}\setminus\{i\}
      & Decision
      \\ \toprule %\midrule %
      $\{1,2,3,4,9\}$ & 1 & $\{2,3,4,9\}$ & $Q_1$ & $0.5$ & Keep 1 \\
      $\{1,2,3,4,9\}$ & 2 & $\{1,3,4,9\}$ & $Q_2$ & 0.0157 & Drop 2 \\ % 0.015625
      $\{1,3,4,9\}$ & 3 & $\{1,4,9\}$ & $Q_2,Q_3$ & 0.0234 & Drop 3 \\ % 0.015625+0.0078125
      $\{1,4,9\}$ & 4 & $\{1,9\}$ & $Q_2,Q_3,Q_4$ & 0.0273 & Drop 4 \\ % 0.015625+0.0078125+0.00390625
      $\{1,9\}$ & 9 & $\{1\}$ & $Q_2,Q_3,Q_4,Q_5$ & $>0.03$ & Keep 9 \\
      \bottomrule
    \end{tabular}
  \end{center}
  \caption{Execution of~\autoref{alg:1axp}} \label{tab:03}
\end{table}

\paragraph{KC languages~\cite{darwiche-jair02}.}
%~\\
Knowledge compilation (KC) languages~\cite{darwiche-jair02} aim at
simplifying queries and transformations in knowledge bases, and have
recently been used as ML models. Concrete examples include binary
decision diagrams~\cite{darwiche-ijcai18,darwiche-aaai19}, among
others~\cite{marquis-kr20,barcelo-nips20,marquis-corr21}.
By noting that the explanation algorithm proposed for DTs exploits
counting of models after conditioning (i.e.\ fixing to a value) a set of
selected features, then we can conclude that, for KC languages that
implement conditioning and model counting in polynomial time, one
min-I$\delta$-relevant set can also be computed in polynomial time. 

%\jnote{Still to be done...}

%% file: algs/oneaxp.tex
\hspace*{\algorithmicindent}
\textbf{Input}: {Classifier $\kappa$, instance $\mbf{v}$, value $\delta$}\\
\hspace*{\algorithmicindent}
\textbf{Output}: {IDRS $\fml{S}$}
\begin{algorithmic}[1]
  \Procedure{$\axp$}{$\kappa,\mbf{v},\delta$}
  \State{$\fml{S} \gets \{1,\ldots,m\}$}
  \For{$i\in\{1,\ldots,m\}$} \Comment{Invariant: $\Prob_{\mbf{x}}(\kappa(\mbf{x})\not=\kappa(\mbf{v}),\mbf{x}_{\fml{S}}=\mbf{v}_{\fml{S}})\le\delta$}
    \State{$\fml{S}\gets\fml{S}\setminus\{i\}$}
    \If{$\Prob_{\mbf{x}}(\kappa(\mbf{x})\not=\kappa(\mbf{v}),\mbf{x}_{\fml{S}}=\mbf{v}_{\fml{S}})>\delta$}
      \State{$\fml{S} \gets \fml{S}\cup\{i\}$}
    \EndIf
  \EndFor
  \State{\bfseries{return}~{$\fml{S}$}}
\EndProcedure
\end{algorithmic}

%% file: res.tex
\section{Experimental Results} \label{sec:res}
This section summarizes the experimental results, which aim at
demonstrating the efficiency of min-I$\delta$-relevant sets if
computed as explanations for DT models trained for various well-known
datasets, over heuristic explanations of
Anchor~\cite{guestrin-aaai18}, both in terms of runtime performance
and explanation precision.

\paragraph{Implementation and benchmarks.}
Min-I$\delta$-relevant sets are computed following the ideas of
\autoref{sec:rsdt} and utilizing the polynomial-time
\autoref{alg:1axp}.
The prototype implementation of the algorithm (\textsc{idrs}) is written in Perl while
the overall experiment is set up and run in Python.\footnote{The
  prototype implementation, benchmarks, instructions and log files of
  the experiment will be made publicly available in the final version
of the paper.}
The precision of the resulting explanations is then assessed using the
generic (and non-monotone) precision metric of
Anchor~\cite{guestrin-aaai18}.
The experiments are conducted on a MacBook Pro with a Dual-Core Intel
Core~i5 2.3GHz CPU with 8GByte RAM running macOS Catalina.

The benchmarks used in the experiment include publicly available and
widely used datasets.  
The datasets originate from UCI %Machine Learning 
ML Repository \cite{uci}
and Penn %Machine Learning 
ML Benchmarks \cite{Olson2017PMLB}.
The number of training instances (resp.\ features) in the target
datasets varies from 3710 to 145585 (resp.\ 12 to 41).
%
%\anote{Are the datasets binarized with OHE? If yes, we need to
%mention it.}
%
All the decision tree models are trained using the learning tool
\emph{ITI} (\emph{Incremental Tree Induction})~\cite{utgoff-ml97,iti}.
Note that the accuracy of all the models is above 73\%, the maximum
depth of the trees varies from 14 to 60 and the total number of 
nodes varies from 49 to 9969.

The experiment was set to iterate over some of the \emph{unique} (see
below) instances of a dataset and to compute an explanation for each
such instance: either a min-I$\delta$-relevant set or an anchor.
As the baseline, we ran Anchor with the default explanation precision
of 0.95.
The prototype implementation \textsc{idrs} was run for the values of $\delta$ from
$\{\textrm{0.05},   \textrm{0.02},  \textrm{0.01}, \textrm{0.0}\}$.
It should be observed that the proposed experiment gives an advantage
to Anchor, as Anchor is allowed to computes explanations guided by its
own metric, whereas I$\delta$-relevant sets \emph{know nothing} about
this metric (which they will be assessed with).

\input{./texfigs/tab-res}

\paragraph{Results.}
\autoref{tab:res} details the results of our
experiment. %\anote{Include the tables.}
First of all, observe that I$\delta$-relevant sets are extremely
simple to compute.
Concretely, the runtime of our prototype explainer \textsc{idrs} normally takes just
\emph{a fraction of a second} per data instance (and never exceeds a
second) to get a subset-minimal I$\delta$-relevant set.
This means that it is at least 1--2 orders of magnitude faster than
Anchor, which can take up to 138 seconds to get a single explanation,
with the average explanation time being up to 10 seconds.
Also observe that the runtime of the proposed approach is not affected
by the value of $\delta$ and tends to be negligible overall.

Second, length-wise I$\delta$-relevant sets also outperform Anchor.
In particular, it is not surprising that the largest
I$\delta$-relevant sets correspond to $\delta$=0 and these on average
include up to 11.4 features.
Explanation size gets further improved when $\delta$ is 0.01, 0.02 or 0.05.
Concretely, it is reduced to \emph{a few} literals per explanation (on
average below 5).
(Also, please refer to the value of standard deviation shown in the
tables.)
On the contrary, Anchor's explanations utilize up to 39 literals, with
the average explanation containing 9 literals. 
These results show an important difference between \textsc{idrs} and
Anchor in terms of interpretability~\cite{miller-pr56}.

Finally and somewhat unexpectedly, %most surprisingly, 
%I$\delta$-relevant sets outperform
\textsc{idrs} outperforms Anchor in terms of explanation precision.
Clearly, the precision of I0-relevant sets (i.e. $\delta$=0) is 100\%,
which demonstrates a significant improvement over anchors.
What is more interesting, however, is that the average precision of
%I$\delta$-relevant sets 
\textsc{idrs}
does not significantly drop down in case of
$\delta\in\{\textrm{0.05}, \textrm{0.02}, \textrm{0.01}\}$.
In particular, 
%their
its 
precision is on par with (or better than) the
explanations provided by Anchor.
%\anote{We may want to emphasize some of the datasets.}
%
%
All the points above confirm that I$\delta$-relevant sets if computed
for DT models provide a viable alternative to Anchor's explanations
from all the considered perspectives, including runtime performance,
explanation size, and precision.

%% file: texfigs/tab-res.tex
\setlength{\tabcolsep}{5pt}
\newcommand{\lpr}{(}
\newcommand{\rpr}{)}

\sisetup{%
  math-rm=\textrm
}

\begin{table*}[t]%[h]
\centering
\resizebox{\textwidth}{!}{
  \begin{tabular}{lcc|c |cccccccccccccccc cccc}
\toprule[1.2pt]
%\rowcolor{white}
\multirow{3}{*}{\bf Dataset} & \multirow{3}{*}{\bf \#F} & \multirow{3}{*}{\bf \#I} & \multirow{3}{*}{\bf \large$\bm\delta$} & \multicolumn{7}{c}{\bf \textsc{idrs} } &  \multicolumn{7}{c}{\bf Anchor} \\
\cmidrule[0.8pt](lr{.75em}){5-11}
\cmidrule[0.8pt](lr{.75em}){12-19}
                             &  \multicolumn{2}{c|}{} &  & \multicolumn{2}{c}{\bf Length} &  \multicolumn{2}{c}{\bf Precision (\%)} &  \multicolumn{3}{c}{\bf Runtime (s)} & \multicolumn{2}{c}{\bf Length} &  \multicolumn{2}{c}{\bf Precision (\%)}  &  \multicolumn{3}{c}{\bf Runtime (s)}  \\  
\cmidrule[0.8pt](lr{.75em}){5-6}
\cmidrule[0.8pt](lr{.75em}){7-8}
\cmidrule[0.8pt](lr{.75em}){9-11}
\cmidrule[0.8pt](lr{.75em}){12-13}
\cmidrule[0.8pt](lr{.75em}){14-15}
\cmidrule[0.8pt](lr{.75em}){16-18}

  %
%\rowcolor{white}
& \multicolumn{2}{c|}{} &  & {\bf M} & {\bf avg} &  {\bf avg} & {\bf dev} & {\bf m }  & {\bf M }  & {\bf avg }   & {\bf M }  & {\bf avg } & {\bf avg} & {\bf dev} & {\bf m }  & {\bf M }  & {\bf avg } \\
\toprule[1.2pt]

\multirow{4}{*}{adult} & \multirow{4}{*}{12} & \multirow{4}{*}{1766} &  0.0  & 10 & 5.1  & 100 & 0.0 & 0.04 & 0.07 & 0.05     & \multirow{4}{*}{12} &  \multirow{4}{*}{5.3} &  \multirow{4}{*}{87.8} &  \multirow{4}{*}{16.7} &  \multirow{4}{*}{0.14} &  \multirow{4}{*}{2.99} &  \multirow{4}{*}{1.20}     \\
 & & & 0.01    & 6 & 3.3 & 85.7 & 20.8 & 0.04 & 0.08 & 0.04   &  &  &  &  &  &  &  &      \\ 
 & & & 0.02  & 6 & 2.8  & 83.0 & 16.4 & 0.04 & 0.08 & 0.05     &  &  &  &  &  &  &  &      \\
 & &  & 0.05    & 5 & 1.9  & 77.7 & 21.0 & 0.04 & 0.11 & 0.06    &  &  &  &  &  &  &  &      \\

\midrule

\multirow{4}{*}{allhyper} & \multirow{4}{*}{29} & \multirow{4}{*}{1113} & 0.0  & 7 & 4.4  & 100 & 0.0 & 0.05 & 0.05 & 0.05    &  \multirow{4}{*}{29} &  \multirow{4}{*}{1.2} &    \multirow{4}{*}{89.5} &  \multirow{4}{*}{7.0} &  \multirow{4}{*}{0.28} &  \multirow{4}{*}{5.75} &  \multirow{4}{*}{0.35}     \\
 & & & 0.01    & 6 & 3.0 & 98.4 & 4.3 & 0.04 & 0.08 & 0.05    &  &  &  &  &  &  &  &      \\
 & & & 0.02  & 6 & 1.0  & 97.7 & 6.3 & 0.05 & 0.07 & 0.05     &  &  &  &  &  &  &  &      \\
 & & & 0.05  & 4 & 1.0  & 97.7 & 6.3 & 0.04 & 0.10 & 0.05     &  &  &  &  &  &  &  &      \\
\midrule

\multirow{4}{*}{ann-thyroid} & \multirow{4}{*}{21} & \multirow{4}{*}{2139} & 0.0  & 10 & 3.9  & 100 & 0.0 & 0.08 & 0.10 & 0.08     &  \multirow{4}{*}{21} &  \multirow{4}{*}{1.3} &   \multirow{4}{*}{96.4} &  \multirow{4}{*}{8.7} &  \multirow{4}{*}{0.22} &  \multirow{4}{*}{8.63} &  \multirow{4}{*}{0.48}     \\
 & & & 0.01    & 6 & 1.4 & 96.9 & 11.4 & 0.07 & 0.13 & 0.08    &  &  &  &  &  &  &  &      \\
 & & & 0.02  & 6 & 1.0  & 96.8 & 11.2 & 0.08 & 0.12 & 0.08     &  &  &  &  &  &  &  &      \\
 & & & 0.05  & 5 & 0.1  & 95.2 & 9.9 & 0.07 & 0.17 & 0.10     &  &  &  &  &  &  &  &      \\
\midrule

\multirow{4}{*}{fars} & \multirow{4}{*}{29} & \multirow{4}{*}{2790} & 0.0  & 15 & 5.9  & 100 & 0.0 & 0.67 & 0.92 & 0.69     &  \multirow{4}{*}{29} &  \multirow{4}{*}{9.0} &    \multirow{4}{*}{73.5} &  \multirow{4}{*}{40.3} &  \multirow{4}{*}{0.30} &  \multirow{4}{*}{57.43} &  \multirow{4}{*}{7.54}     \\
 & & & 0.01    & 6 & 2.0 & 75.2 & 30.9 & 0.58 & 0.81 & 0.69   &  &  &  &  &  &  &  &      \\
 & & & 0.02  & 6 & 2.1  & 70.2 & 35.5 & 0.67 & 0.98 & 0.71    &  &  &  &  &  &  &  &      \\
 & & & 0.05  & 5 & 1.7  & 58.6 & 38.0 & 0.63 & 0.89 & 0.70     &  &  &  &  &  &  &  &      \\
\midrule

\multirow{4}{*}{kddcup} & \multirow{4}{*}{41} & \multirow{4}{*}{4368} & 0.0  & 14 & 11.4  & 100 & 0.0 & 0.44 & 4.14 & 0.46    &  \multirow{4}{*}{39} &  \multirow{4}{*}{2.6}  &  \multirow{4}{*}{23.1} &  \multirow{4}{*}{19.0} &  \multirow{4}{*}{0.42} &  \multirow{4}{*}{137.3} &  \multirow{4}{*}{10.59}     \\
 & & & 0.01    & 8 & 4.4 & 53.7 & 42.9 & 0.42 & 0.84 & 0.45    &  &  &  &  &  &  &  &      \\
 & & & 0.02  & 7 & 4.2  & 51.8 & 22.0 & 0.45 & 0.61 & 0.46     &  &  &  &  &  &  &  &      \\
 & & & 0.05 & 6 & 2.8  & 38.7 & 22.0 & 0.41 & 0.54 & 0.44     &  &  &  &  &  &  &  &      \\

\bottomrule[1.2pt]
\end{tabular}
}
\caption{\footnotesize{Assessing explanations of   I$\delta$-relevant sets (\textsc{idrs}) and comparison with Anchor's explanations.
Columns {\bf \#F} and {\bf \#I} report, resp., the number of features 
and the number of tested instances, in the dataset. 
(Note that for a dataset containing less (resp.\ more) than 10.000 instances,  
30\%  (resp.\ 3\%) of its instances, randomly selected, are used to be explained. 
%(i.e. adult, allhyper and ann-thyroid datasets);
%while for a dataset containing more than 10.000 instances, we randomly picked 
%3\% of its instances, to be explained (i.e. fars and kddcup datasets). 
Moreover, duplicate rows in  the datasets are filtered.) 
Sub-Columns {\bf M} and {\bf avg} of column {\bf Length} show, resp., 
the maximum and average size of an explanation.
Sub-columns  {\bf avg} and {\bf dev} of column {\bf Precision} 
show, resp., the average and standard deviation of the explanation's precision. 
Sub-columns   {\bf m},  {\bf M} and {\bf avg}  of column {\bf Runtime} 
report, resp., minimal, maximal and average time in seconds  
to compute an explanation. } 
}
\label{tab:res}
\end{table*}

%% file: conc.tex
\section{Conclusions}
\label{sec:conc}

$\delta$-relevant sets~\cite{kutyniok-jair21} enable the computation
of approximate (i.e.\ non-universally true) explanations, and reveal
connections between heuristic explanations and non-heuristic
explanations.
A major drawback of $\delta$-relevant sets is their computational complexity.
This paper shows that for DTs, deciding whether there exists a set of
at most $k$ features that $\delta$-relevant is in NP.
Furthermore, the paper proposes relaxed alternative definitions of 
$\delta$-relevant sets, and shows that such alternative definitions
enable the computation of minimal approximate explanations in
polynomial time.
The paper also derives a first result on the duality between
sets of features representing different kinds of (relaxed) $\delta$
relevant sets.
The experimental results, obtained on large DTs learned with a state
of the art tree learner, confirm the practical efficiency and the
quality of explanations when compared with the well-known Anchor
heuristic explainer~\cite{guestrin-aaai18}.

%% file: acks.tex
\paragraph{Acknowledgments.} This work was supported by the AI
Interdisciplinary Institute ANITI, funded by the French program
``Investing for the Future -- PIA3'' under Grant agreement
no.\ ANR-19-PI3A-0004, and by the H2020-ICT38 project COALA
``Cognitive Assisted agile manufacturing for a Labor force supported
by trustworthy Artificial intelligence''.

%% file: xproofs.tex
\section*{Supplementary Material}

\subsection*{Proofs}

\paragraph{Deciding $\pmb{\delta}$-relevancy.}

\begin{definition}[$\majsat$\cite{papadimitriou-bk94}]
  Given a boolean function $f:\{0,1\}^n\to\{0,1\}$, the $\majsat$
  problem is to decide whether the number of points $\mbf{x}$ with 
  $f(\mbf{x})=1$ exceeds the number of points with $f(\mbf{x})=0$.
\end{definition}

It is well-known that $\majsat$ is
$\ppc$-complete~\cite{papadimitriou-bk94}.

\begin{proposition}
  Deciding whether a set $\fml{S}$ is a C$\delta$-relevant set is
  $\ppc$-hard.
\end{proposition}

\begin{Proof}[Sketch]\\
  We reduce $\majsat$ to deciding C$\delta$-relevancy.\\
  Let $f:\{0,1\}^n\to\{0,1\}$ be a boolean function. The variables of
  $f$ are $X=\{x_1,\ldots,x_n\}$. We want to decide whether
  $\prob(f(\mbf{x})=1)>\prob(f(\mbf{x})=0)$.
  % %Let $P=\{p_1,p_2\}$.
  We create another function
  $F:\{0,1\}^{n}\times\{0,1\}^{2}\to\{0,1\}$, such that the variables
  of $F$ are $X$ and $P=\{p_1,p_2\}$. Moreover, $F$ is defined as
  follows:
  \[
  F(\mbf{x},\mbf{p})=\left\{
  \begin{array}{lcl}
    1 & \quad\quad & \tn{if $p_1=p_2=1$} \\
    f(\mbf{x}) & & \tn{otherwise} \\
  \end{array}\right.
  \]
  Set $(\mbf{x}_a,\mbf{p}_a)=((0,\ldots,0),(1,1))$. Clearly,
  $F(\mbf{x}_a,\mbf{p}_a)=1$.\\
  Moreover, set $\delta=0.75$ and pick $S=\{p_1\}$.\\
  Now, if
  $\prob(F(\mbf{x}_b,\mbf{p}_b)=1|(\mbf{x}_b,\mbf{p}_b)_{\fml{S}}=(\mbf{x}_a,\mbf{p}_a)_{\fml{S}})>\delta$
  iff
  the number of points with $f(\mbf{x})=1$ exceeds the number of
  points with $f(\mbf{x})=0$.
\end{Proof}

%% file: paper.bbl
\begin{thebibliography}{10}

\bibitem{marquis-corr21}
G.~Audemard, S.~Bellart, L.~Bounia, F.~Koriche, J.~Lagniez, and P.~Marquis.
\newblock On the computational intelligibility of boolean classifiers.
\newblock {\em CoRR}, abs/2104.06172, 2021.

\bibitem{marquis-kr20}
G.~Audemard, F.~Koriche, and P.~Marquis.
\newblock On tractable {XAI} queries based on compiled representations.
\newblock In {\em KR}, pages 838--849, 2020.

\bibitem{barcelo-nips20}
P.~Barcel{\'{o}}, M.~Monet, J.~P{\'{e}}rez, and B.~Subercaseaux.
\newblock Model interpretability through the lens of computational complexity.
\newblock In {\em NeurIPS}, 2020.

\bibitem{barrett-hmc18}
C.~W. Barrett and C.~Tinelli.
\newblock Satisfiability modulo theories.
\newblock In E.~M. Clarke, T.~A. Henzinger, H.~Veith, and R.~Bloem, editors,
  {\em Handbook of Model Checking}, pages 305--343. Springer, 2018.

\bibitem{berge-bk84}
C.~Berge.
\newblock {\em Hypergraphs: combinatorics of finite sets}.
\newblock Elsevier, 1984.

\bibitem{bjorner-tacas15}
N.~Bj{\o}rner, A.~Phan, and L.~Fleckenstein.
\newblock {\(\nu\)}z - an optimizing {SMT} solver.
\newblock In C.~Baier and C.~Tinelli, editors, {\em TACAS}, pages 194--199,
  2015.

\bibitem{breiman-ss01}
L.~Breiman.
\newblock Statistical modeling: The two cultures.
\newblock {\em Statistical science}, 16(3):199--231, 2001.

\bibitem{lukasiewicz-corr19}
O.~Camburu, E.~Giunchiglia, J.~Foerster, T.~Lukasiewicz, and P.~Blunsom.
\newblock Can {I} trust the explainer? verifying post-hoc explanatory methods.
\newblock {\em CoRR}, abs/1910.02065, 2019.

\bibitem{chinneck-jc91}
J.~W. Chinneck and E.~W. Dravnieks.
\newblock Locating minimal infeasible constraint sets in linear programs.
\newblock {\em {INFORMS} J. Comput.}, 3(2):157--168, 1991.

\bibitem{darwiche-ecai20}
A.~Darwiche and A.~Hirth.
\newblock On the reasons behind decisions.
\newblock In {\em {ECAI}}, pages 712--720, 2020.

\bibitem{darwiche-jair02}
A.~Darwiche and P.~Marquis.
\newblock A knowledge compilation map.
\newblock {\em J. Artif. Intell. Res.}, 17:229--264, 2002.

\bibitem{guidotti-acmcs19}
R.~Guidotti, A.~Monreale, S.~Ruggieri, F.~Turini, F.~Giannotti, and
  D.~Pedreschi.
\newblock A survey of methods for explaining black box models.
\newblock {\em {ACM} Comput. Surv.}, 51(5):93:1--93:42, 2019.

\bibitem{inams-aiia20}
A.~Ignatiev, N.~Narodytska, N.~Asher, and J.~Marques{-}Silva.
\newblock From contrastive to abductive explanations and back again.
\newblock In {\em AI*IA}, 2020.

\bibitem{inms-aaai19}
A.~Ignatiev, N.~Narodytska, and J.~Marques{-}Silva.
\newblock Abduction-based explanations for machine learning models.
\newblock In {\em AAAI}, pages 1511--1519, 2019.

\bibitem{inms-nips19}
A.~Ignatiev, N.~Narodytska, and J.~Marques{-}Silva.
\newblock On relating explanations and adversarial examples.
\newblock In {\em NeurIPS}, pages 15857--15867, 2019.

\bibitem{iti}
{Incremental Decision Tree Induction}.
\newblock \url{https://www-lrn.cs.umass.edu/iti/}, 2020.

\bibitem{iims-corr20}
Y.~Izza, A.~Ignatiev, and J.~Marques{-}Silva.
\newblock On explaining decision trees.
\newblock {\em CoRR}, abs/2010.11034, 2020.

\bibitem{vandenbroeck-ijcai19}
P.~Khosravi, Y.~Liang, Y.~Choi, and G.~{Van den Broeck}.
\newblock What to expect of classifiers? reasoning about logistic regression
  with missing features.
\newblock In {\em IJCAI}, pages 2716--2724, 2019.

\bibitem{lakkaraju-aies20b}
H.~Lakkaraju and O.~Bastani.
\newblock "how do {I} fool you?": Manipulating user trust via misleading black
  box explanations.
\newblock In {\em AIES}, pages 79--85, 2020.

\bibitem{lipton-cacm18}
Z.~C. Lipton.
\newblock The mythos of model interpretability.
\newblock {\em Commun. {ACM}}, 61(10):36--43, 2018.

\bibitem{lundberg-nips17}
S.~M. Lundberg and S.~Lee.
\newblock A unified approach to interpreting model predictions.
\newblock In {\em {NIPS}}, pages 4765--4774, 2017.

\bibitem{kwiatkowska-ijcai21}
E.~L. Malfa, A.~Zbrzezny, R.~Michelmore, N.~Paoletti, and M.~Kwiatkowska.
\newblock On guaranteed optimal robust explanations for {NLP} models.
\newblock In {\em IJCAI}, 2021.
\newblock In press, available from \url{https://arxiv.org/abs/2105.03640}.

\bibitem{miller-pr56}
G.~A. Miller.
\newblock The magical number seven, plus or minus two: Some limits on our
  capacity for processing information.
\newblock {\em Psychological review}, 63(2):81, 1956.

\bibitem{miller-aij19}
T.~Miller.
\newblock Explanation in artificial intelligence: Insights from the social
  sciences.
\newblock {\em Artif. Intell.}, 267:1--38, 2019.

\bibitem{molnar-bk20}
C.~Molnar.
\newblock {\em Interpretable Machine Learning}.
\newblock 2020.
\newblock \url{http://tiny.cc/6c76tz}.

\bibitem{monroe-cacm21}
D.~Monroe.
\newblock Deceiving {AI}.
\newblock {\em Commun. {ACM}}, 64, 2021.

\bibitem{muller-dsp18}
G.~Montavon, W.~Samek, and K.~M{\"{u}}ller.
\newblock Methods for interpreting and understanding deep neural networks.
\newblock {\em Digital Signal Processing}, 73:1--15, 2018.

\bibitem{nsmims-sat19}
N.~Narodytska, A.~A. Shrotri, K.~S. Meel, A.~Ignatiev, and J.~Marques{-}Silva.
\newblock Assessing heuristic machine learning explanations with model
  counting.
\newblock In {\em SAT}, pages 267--278, 2019.

\bibitem{Olson2017PMLB}
R.~S. Olson, W.~La~Cava, P.~Orzechowski, R.~J. Urbanowicz, and J.~H. Moore.
\newblock {PMLB:} a large benchmark suite for machine learning evaluation and
  comparison.
\newblock {\em BioData Mining}, 10(1):36, 2017.

\bibitem{papadimitriou-bk94}
C.~H. Papadimitriou.
\newblock {\em Computational complexity}.
\newblock Addison-Wesley, 1994.

\bibitem{reiter-aij87}
R.~Reiter.
\newblock A theory of diagnosis from first principles.
\newblock {\em Artif. Intell.}, 32(1):57--95, 1987.

\bibitem{guestrin-kdd16}
M.~T. Ribeiro, S.~Singh, and C.~Guestrin.
\newblock "{Why} should {I} trust you?": Explaining the predictions of any
  classifier.
\newblock In {\em KDD}, pages 1135--1144, 2016.

\bibitem{guestrin-aaai18}
M.~T. Ribeiro, S.~Singh, and C.~Guestrin.
\newblock Anchors: High-precision model-agnostic explanations.
\newblock In {\em AAAI}, 2018.

\bibitem{rudin-naturemi19}
C.~Rudin.
\newblock Stop explaining black box machine learning models for high stakes
  decisions and use interpretable models instead.
\newblock {\em Nature Machine Intelligence}, 1(5):206--215, 2019.

\bibitem{xai-bk19}
W.~Samek, G.~Montavon, A.~Vedaldi, L.~K. Hansen, and K.~M{\"{u}}ller, editors.
\newblock {\em Explainable {AI:} Interpreting, Explaining and Visualizing Deep
  Learning}.
\newblock Springer, 2019.

\bibitem{darwiche-ijcai18}
A.~Shih, A.~Choi, and A.~Darwiche.
\newblock A symbolic approach to explaining bayesian network classifiers.
\newblock In {\em IJCAI}, pages 5103--5111, 2018.

\bibitem{darwiche-aaai19}
A.~Shih, A.~Choi, and A.~Darwiche.
\newblock Compiling bayesian network classifiers into decision graphs.
\newblock In {\em {AAAI}}, pages 7966--7974, 2019.

\bibitem{lakkaraju-aies20a}
D.~Slack, S.~Hilgard, E.~Jia, S.~Singh, and H.~Lakkaraju.
\newblock Fooling {LIME} and {SHAP:} adversarial attacks on post hoc
  explanation methods.
\newblock In {\em AIES}, pages 180--186, 2020.

\bibitem{slaney-ecai14}
J.~Slaney.
\newblock Set-theoretic duality: {A} fundamental feature of combinatorial
  optimisation.
\newblock In {\em ECAI}, pages 843--848, 2014.

\bibitem{uci}
{UCI Machine Learning Repository}.
\newblock \url{https://archive.ics.uci.edu/ml}, 2020.

\bibitem{utgoff-ml97}
P.~E. Utgoff, N.~C. Berkman, and J.~A. Clouse.
\newblock Decision tree induction based on efficient tree restructuring.
\newblock {\em Mach. Learn.}, 29(1):5--44, 1997.

\bibitem{kutyniok-jair21}
S.~W{\"{a}}ldchen, J.~MacDonald, S.~Hauch, and G.~Kutyniok.
\newblock The computational complexity of understanding binary classifier
  decisions.
\newblock {\em J. Artif. Intell. Res.}, 70:351--387, 2021.

\bibitem{weld-cacm19}
D.~S. Weld and G.~Bansal.
\newblock The challenge of crafting intelligible intelligence.
\newblock {\em Commun. {ACM}}, 62(6):70--79, 2019.

\end{thebibliography}
